\pdfoutput=1
\documentclass{article}
\usepackage[T1]{fontenc}
\usepackage{times}

\usepackage[left=31mm,right=31mm,top=33mm,bottom=20mm]{geometry}
\usepackage{natbib}
\usepackage{booktabs}       
\usepackage{amsmath,amssymb,amsfonts}
\usepackage[thmmarks, amsmath, standard]{ntheorem}

\usepackage[lined,ruled]{algorithm2e}

\usepackage{enumitem}

\usepackage[usenames,dvipsnames]{color}
\usepackage[colorlinks,
            linkcolor=blue,
            citecolor=blue,
            urlcolor=magenta,
            linktocpage,
            plainpages=false]{hyperref}

\bibpunct{(}{)}{;}{a}{,}{,}

\DeclareMathOperator{\E}{\mathsf{E}}

\DeclareMathOperator*{\opmax}{\vee}

\DeclareMathOperator*{\conv}{conv}

\DeclareMathOperator*{\pipes}{\|}

\newcommand{\loss}{\ell}
\newcommand{\xsloss}{\mathcal{L}_f}
\newcommand{\reals}{\mathbb{R}}
\newcommand{\F}{\mathcal{F}}
\newcommand{\G}{\mathcal{G}}
\newcommand{\X}{\mathcal{X}}
\newcommand{\Y}{\mathcal{Y}}
\newcommand{\Z}{\mathcal{Z}}

\newcommand{\nalg}{n_{\textsc{I}}}
\newcommand{\nboost}{n_{\textsc{II}}}
\newcommand{\Zboost}{\mathbf{Z}_{\textsc{II}}}
\newcommand{\alg}{\mathcal{A}}

\newcommand{\Lip}{L}

\newcommand{\gap}{\textsc{gap}}
\newcommand{\BR}[1]{\textsc{BayesRed}_{\eta} \left( #1, \pi \right)}

\renewtheorem{theorem}{Theorem}
\renewtheorem{lemma}{Lemma}
\renewtheorem{definition}{Definition}
\renewtheorem{corollary}{Corollary}

\title{Fast rates with high probability in exp-concave \\ statistical learning}

\author{Nishant A. Mehta \\ 
             Centrum Wiskunde \& Informatica (CWI) \\
             \texttt{mehta@cwi.nl}}

\date{}

\renewtheorem{theorem}{Theorem}
\renewtheorem{lemma}{Lemma}
\renewtheorem{definition}{Definition}
\renewtheorem{corollary}{Corollary}

\begin{document}

\maketitle

\begin{abstract}
We present an algorithm for the statistical learning setting with a bounded exp-concave loss in $d$ dimensions that obtains excess risk $O(d \log(1/\delta)/n)$ with probability at least $1 - \delta$. 
The core technique is to boost the confidence of recent in-expectation $O(d/n)$ excess risk bounds for empirical risk minimization (ERM), without sacrificing the rate, by leveraging a Bernstein condition which holds due to exp-concavity. 
We also show that with probability $1 - \delta$ the standard ERM method obtains excess risk $O(d (\log(n) + \log(1/\delta))/n)$. 
We further show that a regret bound for any online learner in this setting translates to a high probability excess risk bound for the corresponding online-to-batch conversion of the online learner. 
Lastly, we present two high probability bounds for the exp-concave model selection aggregation problem that are quantile-adaptive in a certain sense. 
The first bound is a purely exponential weights type algorithm, obtains a nearly optimal rate, and has no explicit dependence on the Lipschitz continuity of the loss. 
The second bound requires Lipschitz continuity but obtains the optimal rate.
\end{abstract}

\section{Introduction}

In the statistical learning problem, a learning agent observes a samples of $n$ points $Z_1, \ldots, Z_n$ drawn i.i.d.~from an unknown distribution $P$ over an outcome space $\Z$. The agent then seeks an action $f$ in an action space $\F$ that minimizes their expected loss, or risk, $\E_{Z \sim P} [ \loss(f, Z) ]$, where $\loss$ is a loss function $\loss \colon \F \times \Z \to \reals$. 
Several recent works have studied this problem in the situation where the loss is exp-concave and bounded, 
$\F$ and $\Z$ are subsets of $\reals^d$, 
and $\F$ is convex. 
\cite{mahdavi2015lower} were the first to show that there exists a learner for which, with probability at least $1 - \delta$, the excess risk decays at the rate $d ( \log n + \log(1/\delta)) / n$. 
Via new algorithmic stability arguments applied to empirical risk minimization (ERM), \cite{koren2015fast} and \cite{gonen2016tightening} discarded the $\log n$ factor to obtain a rate of $d/ n$, but their bounds only hold in expectation. 
All three works highlighted the open problem of obtaining a high probability excess risk bound with the rate $d \log(1/\delta) / n$. 
Whether this is possible is far from a trivial question in light of a result of \cite{audibert2008progressive}: when learning over a finite class with bounded $\eta$-exp-concave losses, 
the progressive mixture rule (a Ces\`aro mean of pseudo-Bayesian estimators) with learning rate $\eta$ obtains expected excess risk $O(1/n)$ but, for \emph{any} learning rate, these rules suffer from severe deviations of order $\sqrt{\log (1/\delta) / n}$.

This work resolves the high probability question: we present a learning algorithm with an excess risk bound (Corollary \ref{cor:whp-bounds}) which has rate $d \log(1/\delta) / n$ with probability at least $1 - \delta$. 
ERM also obtains $O((d \log(n) + \log(1/\delta)) / n)$ excess risk, 
a fact that apparently was not widely known although it follows from results in the literature. 
To vanquish the $\log n$ factor with the small $\log(1/\delta)$ price it 
suffices to run a two-phase ERM method based on a confidence-boosting device. 
The key to our analysis is connecting exp-concavity to the \emph{central condition} of \Citet{vanerven2015fast}, which in turn implies a Bernstein condition. 
We then exploit the variance control of the excess loss random variables afforded by the Bernstein condition to \emph{boost} the boosting the confidence trick of \cite{schapire1990strength}.

In the next section, we discuss a brief history of the work in this area. 
In Section \ref{sec:setup-expectation}, we formally define the setting and describe the previous $O(d/n)$ in-expectation bounds. 
We present the results for standard ERM and our confidence-boosted ERM method in Sections \ref{sec:erm} and \ref{sec:confidence-boost} respectively. 
Section \ref{sec:online-to-batch} extends the results of \cite{kakade2009generalization} to exp-concave losses, 
showing that under a bounded loss assumption a regret bound for \emph{any} online exp-concave learner transfers to a high probability excess risk bound 
via an online-to-batch conversion. 
This extension comes at no additional technical price: it is a consequence of the variance control implied by exp-concavity, control leveraged by Freedman's inequality for martingales to obtain a fast rate with high probability. 
This result continues the line of work of \cite{cesa2001generalization} and \cite{kakade2009generalization} and accordingly is about the generalization ability of online exp-concave learning algorithms. 
One powerful consequence of this result is a new guarantee for model selection aggregation: we present a method (Section \ref{sec:ms-aggregate}) for the model selection aggregation problem over finite classes with exp-concave losses that obtains a rate of $O((\log|\F| + \log n) / n)$ with high probability, with no explicit dependence on the Lipschitz continuity of the loss function. All previous bounds of which we are aware have explicit dependence on the Lipschitz continuity of the problem. Moreover, the bound is a quantile-like bound in that it improves with the prior measure on a subclass of nearly optimal hypotheses.

\section{A history of exp-concave learning}
\label{sec:previous}

Learning under exp-concave losses with finite classes dates back to the seminal work of \cite{vovk1990aggregating} and the game of prediction with expert advice, with the first explicit treatment for exp-concave losses due to \cite{kivinen1999averaging}. \cite{vovk1990aggregating} showed that if a game is $\eta$-mixable (which is implied by $\eta$-exp-concavity), one can guarantee that the worst-case individual sequence regret against the best of $K$ experts is at most $\frac{\log K}{\eta}$. An online-to-batch conversion then implies an in-expectation excess risk bound of the same order in the stochastic i.i.d.~setting. 

\cite{audibert2008progressive} showed that when learning over a finite class with exp-concave losses, no progressive mixture rule can obtain a high probability excess risk bound of order better than $\sqrt{\log(1/\delta) / n}$. ERM fares even worse, with a lower bound of $\sqrt{\log |\F| / n}$ \emph{in expectation}.  \citep{juditsky2008learning}. 
\cite{audibert2008progressive} overcame the deviations shortcoming of progressive mixture rules via his \emph{empirical star} algorithm, which first runs ERM on $\F$, obtaining $\hat{f}_{\textsc{erm}}$, and then runs ERM a second time on the star convex hull of $\F$ with respect to $\hat{f}_{\textsc{erm}}$. This algorithm achieves $O(\log|\F| / n)$ with high probability; 
the rate was only proved for squared loss with targets $Y$ and predictions $\hat{y}$ in $[-1, 1]$, but it was claimed that the result can be extended to general, bounded losses $\hat{y} \mapsto \loss(y, \hat{y})$ satisfying smoothness and strong convexity as a function of predictions $\hat{y}$. Under similar assumptions, \cite{lecue2014optimal} proved that a method, $Q$-aggregation, also obtains this rate but can further take into account a prior distribution. 

For convex classes, such as $\F \subset \reals^d$ as we consider here, 
\cite{hazan2007logarithmic} designed the Online Newton Step (ONS) 
and Exponentially Weighted Online Optimization (EWOO) 
algorithms. Both have $O(d \log n)$ regret over $n$ rounds, which, after online-to-batch conversion yields $O(d \log(n) / n)$ excess risk in expectation. 
Until recently, it was unclear whether one could obtain a similar high probability result; however, \cite{mahdavi2015lower} showed that an online-to-batch conversion of ONS enjoys excess risk bounded by $O(d \log(n) / n)$ with high probability. 
While this resolved the statistical complexity of learning up to $\log n$ factors, 
ONS (though efficient) can have a high computational cost of $O(d^3)$ even in simple cases like learning over the unit $\ell_2$ ball, and in general its complexity may be as high as $O(d^4)$ per projection step \citep{hazan2007logarithmic, koren2013open}.

If one hopes to eliminate the $\log n$ factor, the additional hardness of the online setting makes it unlikely that one can proceed via an online-to-batch conversion approach. Moreover, computational considerations suggest circumventing ONS anyways. 
In this vein, as we discuss in the next section both \cite{koren2015fast} and \cite{gonen2016tightening} recently established in-expectation excess risk bounds for a lightly penalized ERM algorithm and ERM itself respectively, without resorting to an online-to-batch conversion. Notably, both works developed arguments based on algorithmic stability, thereby circumventing the typical reliance on chaining-based arguments to discard $\log n$ factors. 
Table \ref{tab:bounds} summarizes what is known and our new results.

\begin{table*}[t]
\small
\centering
\begin{tabular}{l l l l l}
\toprule
& \multicolumn{2}{c}{Convex $\F$\hphantom{10mm}} & \multicolumn{2}{c}{Finite $\F$\hphantom{10mm}} \\ 
\cmidrule(r){2-3}
\cmidrule(l){4-5}
\toprule
Algorithm         & Expectation  & Probability $1 - \delta$ & Expectation & Probability $1 - \delta$ \\ 
\midrule
Progressive mixture & --- & --- & $\log |\F|/n$ & $\Omega(\sqrt{\log(1/\delta) / n})$ \\
Empirical star / $Q$-agg.                  & ---& --- & $\log |\F|/n$ & $(\log |\F| + \log(1/\delta)) / n$ \\
Online Newton Step        & $d \log n/n$  & $d (\log n + \log(1/\delta)) / n$ & --- & --- \\
EWOO                            & $d \log n/n$  & $\boldsymbol{(d \log n + \log(1/\delta)) / n}$ & --- & --- \\
ERM                                 & $d/n$             & $\boldsymbol{(d \log n + \log(1/\delta)) / n}$ & $\Omega(\sqrt{\log|\F|/n})$ & --- \\
Boosted ERM                  &   ---                    & $\boldsymbol{d \log(1/\delta) / n}$ & --- & --- \\
\bottomrule
\end{tabular}
\caption{\label{tab:bounds} Excess risk bounds, with new results in bold. Excluding $\Omega(\cdot)$ bounds, all bounds are big-O upper bounds. Boosted ERM applies \textsc{ConfidenceBoost} to ERM. 
``ERM'' is either penalized ERM \citep{koren2015fast} or ERM \citep{gonen2016tightening}. For simplicity we only show dependence in $d$, $n$, and $\delta$, and we also restrict $Q$-aggregation to uniform prior.}
\end{table*}

\section{Rate-optimal in-expectation bounds}
\label{sec:setup-expectation}

We now describe the setting more formally. 
In this work $\F$ is always assumed to be convex, except in Section \ref{sec:ms-aggregate}, which studies the model selection aggregation problem for countable classes. 
We say a function $A \colon \F \to \reals$ has diameter $C$ if $\sup_{f_1, f_2 \in\F} | A(f_1) - A(f_2) | \leq C$. 
Assume for each $z \in \Z$ that the loss map $\loss(\cdot, z) \colon f \mapsto \loss(f, z)$ is $\eta$-exp-concave, i.e.~$f \mapsto e^{-\eta \loss(f, z)}$ is concave over $\F$. 
We further assume, for each outcome $z$, that the loss $\loss(\cdot, z)$ has diameter $B$. 
We adopt the notation $\loss_f(z) := \loss(f, z)$. 
Given a sample of $n$ points drawn i.i.d.~from an unknown distribution $P$ over $\Z$, our objective is to select a hypothesis $f \in \F$ that minimizes the excess risk $\E_{Z \sim P} [ \loss_f(Z) ] - \inf_{f \in \F} \E_{Z \sim P} [ \loss_f(Z) ]$. We assume that there exists $f^* \in \F$ satisfying $\E [ \loss_{f^*}(Z) ] = \inf_{f \in \F} \E_{Z \sim P} [ \loss_f(Z) ]$; this assumption also was made by \cite{gonen2016tightening} and \cite{kakade2009generalization}.\footnote{This assumption is not explicit from \cite{koren2015fast}, but their other assumptions might imply it. 
Regardless, if their results and those of \cite{gonen2016tightening} hold, 
our analysis in Section \ref{sec:confidence-boost} can be adapted to work if the infimal risk is not achieved, i.e.~if $f^* \in \F$ does not exist.}

Let $\alg_\F$ be an algorithm, defined for a function class $\F$ as a mapping $\alg_\F \colon \bigcup_{n \geq 0} \Z^n \to \F$; we drop the subscript $\F$ when it is clear from the context. 
Our starting point will be an algorithm $\alg$ which, when provided with a sample $\mathbf{Z}$ of $n$ i.i.d.~points, satisfies an expected risk bound of the form
\begin{align} 
\E_{\mathbf{Z} \sim P^n} \left[ \E_{Z \sim P} \left[ \loss_{\alg(\mathbf{Z})}(Z) - \loss_{f^*}(Z) \right] \right] \leq \psi(n) . \label{eqn:expected-bound}
\end{align}
\cite{koren2015fast} and \cite{gonen2016tightening} both established in-expectation bounds of the form \eqref{eqn:expected-bound} that obtain a rate of $O(d/n)$ in the case when $\F \subset \reals^d$, each in a slightly different setting. 
\cite{koren2015fast} assume, for each outcome $z \in \Z$, that the loss $\loss(\cdot, z)$ has diameter $B$ and is $\beta$-smooth for some $\beta \geq 1$, i.e.~for all $f, f' \in \F$, the gradient is $\beta$-Lipschitz:
\begin{align*}
\|\nabla_f \loss(f, z) - \nabla_f \loss(f', z)\|_2 \leq \beta \|f - f'\|_2 .
\end{align*}
They also use a 1-strongly convex regularizer $\Gamma \colon \F \to \reals$ with diameter $R$. Under these assumptions, they show that ERM run with the weighted regularizer $\frac{1}{n} \Gamma$ has expected excess risk at most
\begin{align*}
\psi(n) = \frac{1}{n} \left( \frac{24 \beta d}{\eta} + 100 B d + R \right) .
\end{align*}
It is not known if the smoothness assumption is necessary to eliminate the $\log n$ factor.

\cite{gonen2016tightening} work in a slightly different setting that captures all \emph{known} exp-concave losses. 
They assume that the loss is of the form $\loss_f(z) = \phi_y(\langle f, x \rangle)$, for $\F \subset \reals^d$. They further assume, for each $z = (x,y)$, that the mapping $\hat{y} \mapsto \phi_y(\hat{y})$ is $\alpha$-strongly convex and $\Lip$-Lipschitz, but they do not assume smoothness. 
They show that standard, unregularized ERM has expected excess risk at most 
\begin{align*}
\psi(n) = \frac{2 \Lip^2 d}{\alpha n} = \frac{2 d}{\eta n} ,
\end{align*}
where $\eta = \alpha / \Lip^2$; the purpose of the rightmost expression is that the loss is $\eta$-exp-concave. Although this bound ostensibly is independent of the loss's diameter $B$, the dependence may be masked by $\eta$: for logistic loss, $\eta = e^{-B} /4$, while squared loss admits the more favorable $\eta = 1 / (4 B)^2$.

\section{A high probability bound for ERM}
\label{sec:erm}

As a warm-up to proving a high probability $O(d/n)$ excess risk bound, 
we first show that ERM itself obtains excess risk $O(d \log(n)/n)$ with high probability; 
here and elsewhere, if $\delta$ is omitted the dependence is $\log(1/\delta)$. 
That ERM satisfies such a bound was largely implicit in the literature, and so we make this result explicit. The closest such result, Theorem 1 of \cite{mahdavi2014excess}, does not apply as it relies on an additional assumption (see their Assumption (I)). 
Our assumptions subtly differ from elsewhere in this work. 
We assume that $\F \subset \reals^d$ satisfies $\sup_{f, f' \in \F} \| f - f' \|_2 \leq R$ 
and that, for each outcome $z \in \Z$, the loss $\loss(\cdot, z)$ is $\Lip$-Lipschitz 
and $|\loss_f(z) - \loss_{f^*}(z)| \leq B$. The first two assumptions already imply the last for $B = \Lip R$. 
All these assumptions were made by \cite{mahdavi2014excess} and \cite{koren2015fast}, sometimes implicitly, and while \cite{gonen2016tightening} only make the Lipschitz assumption, for all known $\eta$-exp-concave losses the constant $\eta$ depends on $B$ (which itself typically will depend on $R$).

The first, critical observation is that exp-concavity implies good concentration properties of the excess loss random variable. This is easiest to see by way of the $\eta$-central condition, which the excess loss satisfies. This concept, studied by \Citet{vanerven2015fast} and first introduced by \Citet{vanerven2012mixability} as ``stochastic mixability'', is defined as follows.

\begin{definition}[Central condition]
We say that $(P, \loss, \F)$ satisfies the \emph{$\eta$-central condition} for some $\eta > 0$ if there exists a comparator $f^* \in \F$ such that, for all $f \in \F$,
\begin{align*} 
\E_{Z \sim P} \left[ e^{-\eta (\loss_f(Z) - \loss_{f^*}(Z))} \right] \leq 1 .
\end{align*}
\end{definition}
Jensen's inequality implies that if this condition holds, the corresponding $f^*$ must be a risk minimizer. 
It is known \Citep[Section 4.2.2]{vanerven2015fast} that in our setting $(P, \loss, \F$) satisfies the $\eta$-central condition. 
\begin{lemma} \label{lemma:exp-concavity-to-central}
Let $\F$ be convex. Take $\loss$ to be a loss function $\loss \colon \F \times \Z \to \reals$, and assume that, for each $z \in \Z$, the map $\loss(\cdot, z) \colon f \mapsto \loss(f, z)$ is $\eta$-exp-concave. 
Then, for all distributions $P$ over $\Z$, if there exists an $f^* \in \F$ that minimizes the risk under $P$, then $(P, \loss, \F)$ satisfies the $\eta$-central condition.
\end{lemma}

With the central condition in our grip, Theorem 7 of \cite{mehta2014stochastic} directly implies an $O(d \log(n) / n)$ bound for ERM; however, a far simpler version of that result yields much smaller constants. The proof of the version below, in the appendix for completeness, only makes use of an $(\varepsilon / \Lip)$-net of $\F$ in the $\ell_2$ norm, which induces an $\varepsilon$-net of $\{ \loss_f : f \in \F \}$ in the sup norm.
\begin{theorem} \label{thm:erm-whp-bound}
Let $\F \subset \reals^d$ be a convex set satisfying $\sup_{f, f' \in \F} \| f - f' \|_2 \leq R$. Suppose, for all $z \in \Z$, that the loss $\loss(\cdot, z)$ is $\eta$-exp-concave and $\Lip$-Lipschitz. Let $\sup_{z \in \Z, f \in \F} |\loss_f(z) - \loss_{f^*}(z)| \leq B$. 
Then if $n \geq 5$, with probability at least $1 - \delta$, ERM learns a hypothesis $\hat{f}$ with excess risk bounded as
\begin{align}
\E_{Z \sim P} [ \loss_{\hat{f}}(Z) - \loss_{f^*}(Z) ] 
\leq \frac{1}{n} \left( 8 \left( B \opmax \frac{1}{\eta} \right) \left( d \log(16 \Lip R n) + \log \frac{1}{\delta} \right) + 1 \right) . \label{eqn:erm-whp-bound}
\end{align}
\end{theorem}

\section{Boosting the confidence for high probability bounds}
\label{sec:confidence-boost}

The two existing excess risk bounds mentioned in Section \ref{sec:setup-expectation} decay at the rate $1/n$. A na\"ive application of Markov's inequality unsatisfyingly yields excess risk bounds of order $\psi(n) / \delta$ that hold with probability $1 - \delta$. 
In this section, we present and analyze our meta-algorithm, \textsc{ConfidenceBoost}, which boosts these in-expectation bounds to hold with probability at least $1 - \delta$ at the price of $\log(1/\delta)$ factor. 
This method is essentially the ``boosting the confidence'' trick of \cite{schapire1990strength};\footnote{See also Chapter 4.2 of \cite{kearns1994introduction}.}  
the novelty lies in a refined analysis that exploits a Bernstein-type condition to improve the rate in the final high probability bound from the typical $O(1/\sqrt{n})$ to the desired $O(1/n)$.

Our analysis of \textsc{ConfidenceBoost} actually applies more generally than the exp-concave learning setting, requiring only that $\mathcal{A}$ satisfy an in-expectation bound of the form \eqref{eqn:expected-bound}, the loss $\loss(\cdot, z)$ have bounded diameter for each $z \in \Z$, and the problem $(P, \loss, \F)$ satisfy a \emph{$(C, q)$-Bernstein condition}.

\begin{definition}[Bernstein condition]
We say that $(P, \loss, \F)$ satisfies the \emph{$(C, q)$-Bernstein condition} for some $C > 0$ and $q \in (0, 1]$ if there exists a comparator $f^* \in \F$ such that, for all $f \in \F$,
\begin{align*} 
\E_{Z \sim P} \left[ \left( \loss_f(Z) - \loss_{f^*}(Z) \right)^2 \right] 
\leq C \E_{Z \sim P} \left[ \loss_f(Z) - \loss_{f^*}(Z) \right]^q .
\end{align*}
\end{definition}

Before getting to \textsc{ConfidenceBoost}, we first show that the exp-concave learning setting satisfies the Bernstein condition with the best exponent, $q = 1$, and so is a special case of the more general setting we analyze. 
Recall from Lemma \ref{lemma:exp-concavity-to-central} that the $\eta$-central condition holds for $(P, \loss, \F)$. The next lemma, which adapts a result of \Citet{vanerven2015fast}, shows that the $\eta$-central condition, together with boundedness of the loss, implies that a Bernstein condition holds.
\begin{lemma}[Central to Bernstein] \label{lemma:central-to-bernstein}
Let $X$ be a random variable taking values in $[-B, B]$. Assume that $\E [ e^{-\eta X} ] \leq 1$. 
Then $\E [ X^2 ] \leq 4 \left( 1 / \eta + B \right) \E [ X ]$.
\end{lemma}

\paragraph{Boosting the boosting-the-confidence trick.}
\setlength{\intextsep}{0 pt}
\begin{figure}[t]
    \begin{algorithm}[H]
      \SetAlgoNoLine
      \DontPrintSemicolon
      \KwIn{$\mathbf{Z}_1, \ldots, \mathbf{Z}_K \overset{iid}{\sim} P^{\nalg}$, $\Zboost \sim P^{\nboost}$, learner $\alg_\F$}
      \lFor{$j = 1 \to K$}{$\hat{f}_j = \alg_\F(\mathbf{Z}_j$)}
      \Return $\mathrm{ERM}_{\F_K}(\Zboost)$, with $\F_K = \{\hat{f}_1, \ldots, \hat{f}_K\}$\;
      \caption{\label{alg:boost} \textsc{ConfidenceBoost}}
    \end{algorithm}
\end{figure}

 First, consider running $\alg$ on a sample $\mathbf{Z}_1$ of $n$ i.i.d.~points. The excess risk random variable $\E_Z [ \loss_{\alg(\mathbf{Z}_1)}(Z) - \loss_{f^*}(Z) ]$ is nonnegative, and so Markov's inequality and the expected excess risk being bounded by $\psi(n)$ imply that
\begin{align*} 
\Pr \left( \E_Z [ \loss_{\alg(\mathbf{Z}_1)}(Z) - \loss_{f^*}(Z) ] \geq e \cdot \psi(n) \right) 
\leq \frac{1}{e} .
\end{align*}
Now, let $\mathbf{Z}_1, \ldots, \mathbf{Z}_K$ be independent samples, each of size $n$. Running $\alg$ on each sample yields 
$\hat{f}_1 := \alg(\mathbf{Z}_1), \ldots, \hat{f}_K := \alg(\mathbf{Z}_K)$. 
Applying Markov's inequality as above, combined with independence, implies that with probability at least $1 - e^{-K}$ there exists $j \in [K]$ such that
$\E_{Z \sim P} \bigl[ \loss_{\hat{f}_j}(Z) - \loss_{f^*}(Z) \bigr] \leq e \cdot \psi(n)$. 
Let us call this good event \textsc{good}.

Our quest is now to show that on event \textsc{good}, we can identify any of the hypotheses $\hat{f}_1, \ldots, \hat{f}_K$ approximately satisfying $\E_{Z \sim P} \bigl[ \loss_{\hat{f}_j}(Z) - \loss_{f^*}(Z) \bigr] \leq e \cdot \psi(n)$, where by ``approximately'' we mean up to some slack that weakens the order of our resulting excess risk bound by a multiplicative factor of at most $K$. 
As we will see, it suffices to run ERM over this finite subclass using a fresh sample. 
The proposed meta-algorithm is presented in Algorithm \ref{alg:boost}.

\paragraph{Analysis.}

From here on out, we treat the initial sample of size $K n$ as fixed and unhat the $K$ estimators above, referring to them as $f_1, \ldots, f_K$. Without loss of generality, we further assume that they are sorted in order of increasing risk (breaking ties arbitrarily). 
Our goal now is to show that running ERM on the finite class $\F_K := \{f_1, \ldots, f_K\}$ yields low excess risk with respect to comparator $f_1$. 
A typical analysis of the boosting the confidence trick would apply Hoeffding's inequality to select a risk minimizer optimal to resolution $1/\sqrt{n}$, but this is not good enough here. 
As a further boost to the trick, this time with respect to its resolution, we will establish that a Bernstein condition holds over a particular subclass of $\F_K$ with high probability, which will in turn imply that ERM obtains $O(1/n^{1/(2 - q)})$ excess risk over $\F_K$.

We first establish an \emph{approximate} Bernstein condition for $(P, \loss, \F_K)$. 
Since 
$\| \loss_{f_j} - \loss_{f_1} \|_{L_2(P)} 
\leq \| \loss_{f_j} - \loss_{f^*} \|_{L_2(P)} 
       + \| \loss_{f_1} - \loss_{f^*} \|_{L_2(P)}$ 
for all $f_j \in \F_K$, 
from the $(C, q)$-Bernstein condition, 
\begin{align*}
\| \loss_{f_j} - \loss_{f_1} \|_{L_2(P)}^2 
&\leq C \left( \E [ \loss_{f_j} - \loss_{f^*} ]^q 
                       + \E [ \loss_{f_1} - \loss_{f^*} ]^q 
                       + 2 \left( \E [ \loss_{f_j} - \loss_{f^*} ] \cdot \E [ \loss_{f_1} - \loss_{f^*} ] \right)^{q/2} 
             \right) \\
&\leq C \bigl( 3 \E [ \loss_{f_j} - \loss_{f^*} ]^q
                      + \E [ \loss_{f_1} - \loss_{f^*} ]^q \bigr) \\
&= C \bigl( 3 \left( \E [ \loss_{f_j} - \loss_{f_1} ] + \E [ \loss_{f_1} - \loss_{f^*} ] \right)^q
                      + \E [ \loss_{f_1} - \loss_{f^*} ]^q \bigr) \\
&\leq C \bigl( 3 \E [ \loss_{f_j} - \loss_{f_1} ]^q 
                       + 4 \E [ \loss_{f_1} - \loss_{f^*} ]^q 
             \bigr) ;
\end{align*}
where the last step follows because the map $x \mapsto x^q$ is concave and hence subadditive. 
We call this bound an approximate Bernstein condition because, on event \textsc{good}, for all $f_j \in \F_K$:
\begin{align*}
\| \loss_{f_j} - \loss_{f_1} \|_{L_2(P)}^2 \leq C \left( 3 \E [ \loss_{f_j} - \loss_{f_1} ]^q + 4 (e \cdot \psi(n))^q \right) .
\end{align*}
Define the class $\F_K'$ as the set $\{f_1\} \cup \left\{ f_j \in \F_K : \E [ \loss_{f_j} - \loss_{f_1} ] \geq 4^{1/q} e \cdot \psi(n) \right\}$. 
Then with probability $\Pr(\textsc{good}) \geq 1 - e^{-K}$, the problem $(P, \loss, \F_K')$ satisfies the $(4 C, q)$-Bernstein condition. 

We now analyze the outcome of running ERM on $\{f_1, \ldots, f_k\}$ using a fresh sample of $n$ points. The next lemma shows that ERM performs favorably under a Bernstein condition, a well-known result.

\begin{lemma} \label{lemma:erm-bernstein}
Let $\G$ be a finite class of functions $\{f_1, \ldots, f_K\}$ and assume without loss of generality that $f_1$ is a risk minimizer. Let $\G' \subset \G$ be a subclass for which, for all $f \in \G'$:
\begin{align*}
\E [ (\loss_f - \loss_{f_1})^2 ] \leq C \E [ \loss_f - \loss_{f_1} ]^q ,
\end{align*}
and $\loss_f - \loss_{f_1} \leq B$ almost surely. 
Then, with probability at least $1 - \delta$, ERM run on $\G$ will \emph{not} select any function $f$ in $\G'$ whose excess risk satisfies
\begin{align*}
\E [ \loss_f - \loss_{f_1} ] 
\geq \left( \frac{2 \left( C + \frac{B^{2-q}}{3} \right) \log \frac{|\G'| - 1}{\delta}}{n} \right)^{1/(2 - q)} .
\end{align*}
\end{lemma}

Applying Lemma \ref{lemma:erm-bernstein} with $\G = \F_K$ and $\G' = \F_K'$, with probability at least $1 - \delta$ over the fresh sample, ERM selects a function $f_j$ falling in one of two cases:
\begin{itemize}
\item $\E_{Z \sim P} [ \loss_{f_j}(Z) - \loss_{f_1}(Z) ] \leq 4^{1/q} e \cdot \psi(n)$;
\item $\E_{Z \sim P} [ \loss_{f_j}(Z) - \loss_{f_1}(Z) ] \leq \left( \frac{2 \left( C + \frac{B^{2-q}}{3} \right) \log \frac{K}{\delta}}{n} \right)^{1 / (2 - q)}$ (using $|\F_K'| - 1 \leq K$).
\end{itemize}
We now run \textsc{ConfidenceBoost} with $K = \lceil \log(2/\delta) \rceil$ on a sample of $n$ points, with $\nalg = \frac{n}{2 K}$ and $\nboost = \frac{n}{2}$; for simplicity, we assume that $2 K$ divides $n$. Taking the failure probability for the ERM phase to be $\delta/2$, \textsc{ConfidenceBoost} admits the following guarantee.

\begin{theorem} \label{thm:whp-bound}
Let $(P, \loss, \F)$ satisfy the $(C, q)$-Bernstein condition, and assume for all $z \in \Z$ that the loss $\loss(\cdot, z)$ has diameter $B$. Impose any necessary assumptions such that algorithm $\alg$ obtains a bound of the form \eqref{eqn:expected-bound}. 
Then, with probability at least $1 - \delta$, \textsc{ConfidenceBoost} run with $K = \lceil \log(2 / \delta) \rceil$, $\nalg = n / (2 K)$, and $\nboost = n / 2$ learns a hypothesis $\hat{f}$ with excess risk 
$\E_{Z \sim P} [ \loss_{\hat{f}}(Z) - \loss_{f^*}(Z) ]$ at most
\begin{align}
e \cdot \psi \left( \frac{n}{2 \left\lceil \log \frac{2}{\delta} \right\rceil} \right) 
+ \max \left\{ 
        4^{1/q} e \cdot \psi \left( \frac{n}{2 \log \left\lceil \frac{2}{\delta} \right\rceil } \right), 
       \left( \frac{4 \bigl( C + \frac{B^{2-q}}{3} \bigr) \left( \log \frac{1}{\delta} + \log \lceil \log \frac{2}{\delta} \rceil \right)}{n} \right)^{1 / (2 - q)} 
   \right\} . \label{eqn:meta-bound} 
\end{align}
\end{theorem}

The next result for exp-concave learning is immediate.
\begin{corollary} \label{cor:whp-bounds}
Applying Theorem \ref{thm:whp-bound} with $\alg_\F$ the algorithm of \cite{koren2015fast} and their assumptions (with $\beta \geq 1$), the bound in Theorem \ref{thm:whp-bound} specializes to
\begin{align}
O \left( \frac{\log \frac{1}{\delta}}{n} \left(  
    \frac{d \beta}{\eta} + d B + R \right) \right) . \label{eqn:boosted-koren-bound}
\end{align}
Similarly taking $\alg_\F$ the algorithm of \cite{gonen2016tightening} and their assumptions yields
\begin{align}
O \left( 
    \frac{\log \frac{1}{\delta}}{n} \left( \frac{d}{\eta} + B \right) 
\right) . \label{eqn:boosted-gonen-bound}
\end{align}
\end{corollary}

\paragraph{Remarks.}
As we saw from Lemmas \ref{lemma:exp-concavity-to-central} and \ref{lemma:central-to-bernstein}, in the exp-concave setting a Bernstein condition holds for the class $\F$. A natural inquiry is if one could use this Bernstein condition to show directly a high probability fast rate of $O(d/n)$ for ERM. Indeed, under strong convexity (which is strictly stronger than exp-concavity), \cite{sridharan2009fast} show that a similar bound for ERM is possible; however, 
they used strong convexity to bound a localized complexity. 
It is unclear if exp-concavity can be used to bound a localized complexity, and the Bernstein condition alone seems insufficient; such a bound may be possible via ideas from the local norm analysis of \cite{koren2015fast}. 
While we think controlling a localized complexity from exp-concavity is a very interesting and worthwhile direction, we leave this to future work, and for now only conjecture that ERM also enjoys excess risk bounded by $O((d + \log(1/\delta)) /n)$ with high probability. This conjecture is from analogy to the empirical star algorithm of \cite{audibert2008progressive}, which for convex $\F$ reduces to ERM itself; note that the conjectured effect of $\log(1/\delta)$ is additive rather than multiplicative.

\section{Online-to-batch-conversion}
\label{sec:online-to-batch}

The present section's purpose is to show that if one is willing to accept the additional $\log n$ factor in a high probability bound, then it is sufficient to use an online-to-batch conversion of an online exp-concave learner whose worst-case cumulative regret (over $n$ rounds) is logarithmic in $n$. 
Using such a conversion, it is easy to get an excess risk bound with the additional $\log n$ factor that holds \emph{in expectation}. The key difficulty is making such a bound hold with high probability. 
This result provides an alternative to the high probability $O(\log n / n)$ result for ERM in Section \ref{sec:erm}. 

\cite{mahdavi2015lower} previously considered an online-to-batch conversion of ONS and established the first explicit high probability  $O(\log n / n)$ excess risk bound in the exp-concave statistical learning setting. 
Their analysis is elegant but seems to be intimately coupled to ONS; it consequently is unclear if their analysis can be used to grasp excess risk bounds by online-to-batch conversions of other online exp-concave learners. This leads to our next point and a new path: it is possible to transfer regret bounds to high probability excess risk bounds via online-to-batch conversion for general online exp-concave learners. Our analysis builds strongly on the analysis of \cite{kakade2009generalization} in the strongly convex setting.


We first consider a different, related setting: online convex optimization (OCO) under a $B$-bounded, $\nu$-strongly convex loss that is $\Lip$-Lipschitz with respect to the action. 
An OCO game unfolds over $n$ rounds. An adversary first selects a sequence of $n$ convex loss functions $c_1, \ldots, c_n$. In round $t$, the online learner plays $f_t \in \F$, the environment subsequently reveals cost function $c_t$, and the learner suffers loss $c_t(f_t)$. 
Note that the adversary is oblivious, and so the learner does not necessarily need to randomize. 
Because we are interested in analyzing the statistical learning setting, we constrain the adversary to play a sequence of $n$ points $z_1, \ldots, z_n \in \Z$, inducing cost functions $\loss(\cdot, z_1), \ldots, \loss(\cdot, z_n)$. 

Consider an online learner that sequentially plays actions $f_1, \ldots, f_n \in \F$ in response to $z_1, \ldots, z_n$, so that $f_t$ depends on $(z_1, \ldots, z_{t-1})$. The (cumulative) regret is defined as
\begin{align*}
\sum_{t=1}^n \loss_{f_t}(z_t) - \inf_{f \in \F} \sum_{t=1}^n \loss_f(z_t) .
\end{align*}
When the losses are bounded, strongly convex, and Lipschitz,  \cite{kakade2009generalization} showed that if an online algorithm has regret $\mathcal{R}_n$ on an i.i.d.~sequence $Z_1, \ldots, Z_n \sim P$, 
online-to-batch conversion by simple averaging of the iterates 
$\bar{f}_n := \frac{1}{n} \sum_{t=1}^n f_t$ admits the following guarantee.
\begin{theorem}[Cor.~5, \cite{kakade2009generalization}]
For all $z \in \Z$, assume that $\loss(\cdot, z)$ is bounded by $B$, $\nu$-strongly convex, and $\Lip$-Lipschitz. Then with probability at least $1 - 4 \log(n) \delta$ the action $\bar{f}_n$ satisfies excess risk bound
\begin{align*}
\E_{Z \sim P} [ \loss_{\bar{f}_n}(Z) - \loss_{f^*}(Z) ] 
\leq \frac{\mathcal{R}_n}{n} + 4 \sqrt{\frac{\Lip^2 \log \frac{1}{\delta}}{\nu}} \frac{\sqrt{\mathcal{R}_n}}{n} + \max \left\{ \frac{16 \Lip^2}{\nu}, 6 B \right\} \frac{\log \frac{1}{\delta}}{n} .
\end{align*}
\end{theorem}

Under various assumptions, there are OCO algorithms that obtain worst-case regret (under all sequences $z_1, \ldots, z_n$) $\mathcal{R}_n = O(\log n)$. For instance, Online Gradient Descent \citep{hazan2007logarithmic} admits the regret bound $\mathcal{R}_n \leq \frac{G^2}{2 \nu} (1 + \log n)$, where $G$ is an upper bound on the gradient.

What if we relax strong convexity to exp-concavity? As we will see, it is possible to extend the analysis of \cite{kakade2009generalization} to $\eta$-exp-concave losses. Of course, such a regret-to-excess-risk bound conversion is useful only if we have online algorithms and regret bounds to start with. Indeed, at least two such algorithms and bounds exist, due to \cite{hazan2007logarithmic}:
\begin{itemize}[leftmargin=4.8mm,itemsep=2pt]
\item ONS, 
with $\mathcal{R}_n \leq 5 \left( \frac{1}{\eta} + G D \right) d \log n$, 
where $G$ is a bound on the gradient and $D$ is a bound on the diameter of the action space.
\item Exponentially Weighted Online Optimization (EWOO), 
with $\mathcal{R}_n \leq \frac{1}{\eta} d \left( 1 + \log(n + 1) \right)$. 
The better regret bound comes at the price of not being computationally efficient. 
EWOO \emph{can} be run in randomized polynomial time, but the regret bound then holds only in expectation (which is insufficient for an online-to-batch conversion).
\end{itemize}

We now show how to extend the analysis of \cite{kakade2009generalization} to exp-concave losses. While similar results can be obtained from the work of \cite{mahdavi2015lower} for the specific case of ONS, our analysis is agnostic of the base algorithm. 
A particular consequence is that our analysis also applies to EWOO, which, although highly impractical, offers a better regret bound. Moreover, our analysis applies to any future online learning algorithms which may have improved guarantees and computational complexities. 
The key insight is that exp-concavity implies a variance inequality similar to Lemma 1 of \cite{kakade2009generalization}, a pivotal result of that work that unlocks Freedman's inequality for martingales \citep{freedman1975tail}. Let $Z_1^t$ denote the sequence $Z_1, \ldots, Z_t$.

\begin{lemma}[Conditional variance control]
Define the Martingale difference sequence
\begin{align*}
\xi_t 
:= \E_Z \bigl[ \loss_{f_t}(Z) - \loss_{f^*}(Z) \bigr] 
     - \bigl( \loss_{f_t}(Z_t) - \loss_{f^*}(Z_t) \bigr) .
\end{align*}
\begin{flalign*}
\text{Then} && &&
\mathrm{Var} \left[ \xi_t \mid Z_1^{t-1} \right] 
\leq 4 \left( \frac{1}{\eta} + B \right) \E_Z \bigl[ \loss_{f_t}(Z) - \loss_{f^*}(Z) \bigr] . && && &&
\end{flalign*}
\end{lemma}
\begin{proof}
Observe that 
$\mathrm{Var} \left[ \xi_t \mid Z_1^{t-1} \right] 
= \mathrm{Var} \bigl[ \loss_{f_t}(Z_t) - \loss_{f^*}(Z_t) \mid Z_1^{t-1} \bigr]$. 
Treating the sequence $Z_1^{t-1}$ as fixed and also treating $f_t$ as a fixed parameter $f$, the above conditional variance equals 
$\mathrm{Var} \bigl[ \loss_f(Z) - \loss_{f^*}(Z) \bigr]$, 
where the randomness lies entirely in $Z \sim P$. 
Then, Lemma \ref{lemma:central-to-bernstein} implies that
$\mathrm{Var} \bigl[ \loss_f(Z) - \loss_{f^*}(Z) \bigr] 
\leq 4 \left( \frac{1}{\eta} + B \right) \E \left[ \loss_f(Z) - \loss_{f^*}(Z) \right]$.
\end{proof}

The next corollary is from a retrace of the proof of Theorem 2 of \cite{kakade2009generalization}.
\begin{corollary} \label{cor:online-to-batch}
For all $z \in \Z$, let $\loss(\cdot, z)$ be bounded by $B$ and $\eta$-exp-concave with respect to the action $f \in \F$. Then with probability at least $1 - \delta$, for any $n \geq 3$, the excess risk of $\bar{f}_n$ is at most
\begin{align*}
\frac{\mathcal{R}_n}{n} 
+ 4 \sqrt{ \left( \frac{1}{\eta} + B \right) \log \frac{4 \log n}{\delta}} \cdot \frac{\sqrt{\mathcal{R}_n}}{n} 
+ 16 \left( \frac{1}{\eta} + B \right) \frac{\log \frac{4 \log n}{\delta}}{n} .
\end{align*}
\end{corollary}
In particular, an online-to-batch conversion of EWOO yields excess risk of order
\begin{align*}
\frac{d \log n}{\eta n} 
+ \frac{\sqrt{d \log n}}{n} \left(
      \sqrt{\frac{(\log \log n) B}{\eta}} 
      + \sqrt{\left( \frac{1}{\eta^2} + \frac{B}{\eta} \right) \log \frac{1}{\delta}} \right) 
+ \frac{(\log \log n) B + B \log \frac{1}{\delta}}{n} .
\end{align*}
By proceeding similarly one can get a guarantee for ONS, under the additional assumptions that $\F$ has bounded diameter and that, for all $z \in \Z$, the gradient $\nabla_f \loss(f, z)$ has bounded norm.

\paragraph{Obtaining $\boldsymbol{o(\log n)}$ excess risk.}
The worst-case regret bounds in this online setting have a $\log n$ factor, but when the environment is stochastic and the distribution satisfies some notion of easiness the actual regret can be $o(\log n)$. In such situations the excess risk similarly can be $o(\log n)$ because our excess risk bounds depend not on worst-case regret bounds but rather the actual regret. We briefly explore one scenario where such improvement is possible. 
Suppose that the loss is also $\beta$-smooth; then, in situations when the cumulative loss of $f^*$ is small, the analysis of \citet[Theorem 1]{orabona2012beyond} for ONS yields a more favorable regret bound: they show a regret bound of order $\log \bigl( 1 + \sum_{t=1}^n \loss_{f^*}(Z_t) \bigr)$. As a simple example, consider the case when the problem is realizable in the sense that $\loss_{f^*}(Z) = 0$ almost surely. Then the regret bound is constant and the rate with respect to $n$ for the excess risk in Corollary \ref{cor:online-to-batch} is $\frac{\log \log n}{n}$.

\section{Model selection aggregation}
\label{sec:ms-aggregate}

In the model selection aggregation problem for exp-concave losses, we are given a countable class $\F$ of functions from an input space $\X$ to an output space $\Y$ and a loss $\loss \colon \Y \times \Y \to \reals$; for each $y \in \Y$, the mapping $\hat{y} \mapsto \loss(y, \hat{y})$ is $\eta$-exp-concave. The loss is a supervised loss, as in supervised classification and regression, unlike the more general loss functions used in the rest of the paper which fit into Vapnik's general setting of the learning problem \citep{vapnik1995nature}. 
The random points $Z \sim P$ now decompose into an input-output pair $Z = (X, Y) \in \Z = \X \times \Y$. 
We often use the notation $\loss_f(Z) := \loss(Y, f(X))$.
The goal is the same as in the stochastic exp-concave optimization problem, but now $\F$ fails to be convex (and the exp-concavity assumption slightly differs).

After \cite{audibert2008progressive} showed that the progressive mixture rule cannot obtain fast rates with high probability, several works developed methods that departed from progressive mixture rules and gravitated instead toward ERM-style rules, starting with the empirical star algorithm of
\cite{audibert2008progressive} and a subsequent method of \cite{lecue2009aggregation} which runs ERM over the convex hull of a data-dependent subclass. \cite{lecue2014optimal} extended these results to take into account a prior on the class using their $Q$-aggregation procedure. All the methods require Lipschitz continuity of the loss\footnote{\cite{audibert2008progressive} only proved the case of bounded squared loss with a suggestion for how to handle the case of exp-concave losses; because of the techniques used, it is likely that Lipschitz continuity would come into play.} and are for finite classes, although we believe that $Q$-aggregation combined with a suitable prior extends to countable classes. 
In this section, we present an algorithm that carefully composes exponential weights-type algorithms and still obtains a fast rate with high probability for the model selection aggregation problem. One incarnation can do so with the fast rate of $O(\log |\F| / n)$ for finite $|\F|$, by relying on Boosted ERM. Another, ``pure'' version is based on exponential weights-type procedures alone, can get a rate of $O(\log |\F| / n + \log n / n)$ with no explicit dependence on the Lipschitz continuity of the loss. To our knowledge, this is the first fast rate high probability bound for model selection aggregation that lacks explicit dependence on the Lipschitz constant of the loss. Both results hold more generally, allowing for countable classes, taking into account a prior distribution $\pi$ over $\F$, and providing a quantile-like improvement when there is a low quantile with close to optimal risk.

Since $\F$ is countable and hence not convex, algorithms for stochastic exp-concave optimization do not directly apply. 
Our approach is to apply stochastic exp-concave optimization to the convex hull of a certain small cardinality and data-dependent subset of $\F$. The first phase of obtaining this subset makes use of the progressive mixture rule. We offer two variants for the second phase: \textsc{PM-EWOO} (Algorithm \ref{alg:pm-ewoo}) and \textsc{PM-CB} (Algorithm \ref{alg:pm-cb}). In the algorithms, $\alg^{\mathsf{pm}}$ and $\alg^{\mathsf{ew}}$ are online-to-batch conversions of the progressive mixture rule and EWOO respectively, $\alg^{\mathsf{cb}}$ is \textsc{ConfidenceBoost}, and $\alg^{\mathsf{erm}}$ is ERM.

Our interest in \textsc{PM-EWOO} is two-fold: 
  \emph{(i)} it is a ``purely'' exponential weights type method in that it is based only on the progressive mixture rule and EWOO;
  \emph{(ii)} it does not require any Lipschitz assumption on the loss function, unlike all previous work.
\begin{theorem} \label{thm:pm-ewoo}
Let $\F$ be a countable and $\pi$ a prior distribution over $\F$. Assume that for each $y$ the loss $\loss \colon \hat{y} \mapsto \loss(y, \hat{y})$ is $\eta$-exp-concave. Further assume that $\sup_{f,f' \in \F} |\loss(y, f(x)) - \loss(y, f'(x))| \leq B$ for all $(x, y)$ in the support of $P$. 
Then with probability at least $1 - \delta$, \textsc{PM-EWOO} run with $K = \lceil \log(2/\delta) \rceil$, $\nalg = n / (2 K)$ and $\nboost = n / 2$ learns a hypothesis $\hat{f}$ satisfying
\begin{align*}
\E_{Z \sim P} \left[ \loss_{\hat{f}}(Z) - \loss_{f^*}(Z) \right] 
\leq e \cdot \BR{\frac{n}{2 \lceil \log \frac{2}{\delta} \rceil}} 
       + \theta_{\textsc{ew}}(\delta, n) ,
\end{align*}
\begin{flalign*}
\text{with} && && \theta_{\textsc{ew}}(\delta, n) = 
O \left( 
                     \frac{\sqrt{B} \left( \log \frac{1}{\delta} + \sqrt{\log \frac{1}{\delta}} \log n \right)}
                             {\eta n} 
                    + \frac{B \log \frac{\log n}{\delta})}{n} 
                \right) . && && &&
\end{flalign*}
\end{theorem}
Here, $\BR{n}$ is the 
\emph{$\eta$-generalized Expected Bayesian Redundancy}  \citep{takeuchi1998robustly,grunwald2012safe}, 
defined as
\begin{align*}
\inf_{\rho \in \Delta(\F)} \left\{ 
        \E_Z \left[ \E_{f \sim \rho} \left[ \loss_f(Z) \right] - \loss_{f^*}(Z) \right] 
        + \frac{D(\rho \pipes \pi)}{\eta (n + 1)}
\right\} ,
\end{align*}
for $D(\cdot \pipes \cdot)$ the KL-divergence. 
The bound can be rewritten as a quantile-like bound; for all $\rho \in \Delta(\F)$:
\begin{align*}
\E_{Z \sim P} \left[ \loss_{\hat{f}}(Z) - \E_{f \sim \rho} \left[ \loss_f(Z) \right] \right] 
\leq (e - 1) \gap(\rho, f^*) 
          + \frac{2 e \left\lceil \log \frac{2}{\delta} \right\rceil D(\rho \pipes \pi)}{\eta n} 
          + \theta_{\textsc{ew}}(\delta, n) ,
\end{align*}
where $\gap(\rho, f^*) := \E_Z \left[ \E_{f \sim \rho} \left[ \loss_f(Z) \right] - \loss_{f^*}(Z) \right]$. This bound enjoys a quantile-like improvement when $\gap(\rho, f^*)$ is small. For instance, if there is a set $\F'$ of large prior measure which has excess risk close to $f^*$, then Theorem \ref{thm:pm-ewoo} pays $\log (1/\pi(\F'))$ for the complexity; in contrast, Theorem A of \cite{lecue2014optimal} pays a higher complexity price of $\log (1/\pi(f^*))$.

\begin{figure}[t]
\begin{algorithm}[H]
  \SetAlgoNoLine
  \DontPrintSemicolon
  \KwIn{$\mathbf{Z}_1, \ldots, \mathbf{Z}_K \overset{iid}{\sim} P^{\nalg}$, $\Zboost \sim P^{\nboost}$}
  \lFor{$j = 1 \to K$}{$\hat{f}_j = \alg^{\mathsf{pm}}_\F(\mathbf{Z}_j$)}
  \Return $\alg^{\mathsf{ew}}_{\F_K}(\Zboost)$, with $\F_K = \conv(\{\hat{f}_1, \ldots, \hat{f}_K)\}$\;
  \caption{\label{alg:pm-ewoo} \textsc{PM-EWOO}}
\end{algorithm}
\end{figure}

Lastly, we provide a simpler bound by specializing to the case of $\rho$ concentrated entirely on $f^*$. Then
\begin{align*}
\E_{Z \sim P} \left[ \loss_{\hat{f}}(Z) - \loss_{f^*}(Z) \right] 
\leq \frac{2 e \left\lceil \log \frac{2}{\delta} \right\rceil \log \frac{1}{\pi(f^*)}}{\eta n}
       + \theta_{\textsc{ew}}(\delta, n) .
\end{align*}
Theorem \ref{thm:pm-ewoo} does not explicitly require Lipschitz continuity of the loss, but the rate is suboptimal due to the extra $\log n$ factor. 
The next result obtains the correct rate by using \textsc{ConfidenceBoost} for the second stage of the procedure.
\begin{theorem} \label{thm:pm-cb}
Take the assumptions of Theorem \ref{thm:pm-ewoo}, but instead assume that for each $y$ the loss $\loss \colon \hat{y} \mapsto \loss(y, \hat{y})$ is $\alpha$-strongly convex and $\Lip$-Lipschitz (so $(\alpha / \Lip^2)$-exp-concavity holds). 
Then with probability at least $1 - \delta$, \textsc{PM-CB} run with $K = \lceil \log(3/\delta) \rceil$, $\nalg = n / (4 K)$ and $\nboost = n / 2$ learns a hypothesis $\hat{f}$ satisfying
\begin{align*}
\E_{Z \sim P} \left[ \loss_{\hat{f}}(Z) - \loss_{f^*}(Z) \right] 
\leq e \cdot \BR{\frac{n}{4 \lceil \log \frac{3}{\delta} \rceil}} 
       + \theta_{\textsc{cb}}(\delta, n) ,
\end{align*}
\begin{flalign*}
\text{with} && \theta_{\textsc{cb}}(\delta, n) = O \left( \frac{\left(\log \frac{1}{\delta}\right)^2}{\eta n} + \frac{B \log \frac{1}{\delta}}{n} \right) . && && &&
\end{flalign*}
\end{theorem}

\begin{figure}[t]
\begin{algorithm}[H]
  \SetAlgoNoLine
  \DontPrintSemicolon
  \KwIn{$\mathbf{Z}_1, \ldots, \mathbf{Z}_{2 K} \overset{iid}{\sim} P^{\nalg}$, $\Zboost \sim P^{\nboost}$}
  \lFor{$j = 1 \to K$}{$\hat{f}_j = \alg^{\mathsf{pm}}_\F(\mathbf{Z}_j$)}
   \Return $\alg^{\mathsf{cb}}(\mathbf{Z}_{k + 1}, \ldots, \mathbf{Z}_{2 k}, \Zboost, \alg^{\mathsf{erm}}_{\F_k})$, with $\F_k = \conv(\{\hat{f}_1, \ldots, \hat{f}_k)\}$\;
  \caption{\label{alg:pm-cb} \textsc{PM-CB}}
\end{algorithm}
\end{figure}

The proofs of Theorems \ref{thm:pm-ewoo} and \ref{thm:pm-cb} are nearly identical and left to the appendix. 
We sketch a proof here, as it uses a novel reduction of the second phase to a low-dimensional stochastic exp-concave optimization problem. For simplicity, we restrict to the case of finite $\F$, uniform prior $\pi$, and competing with $f^*$. 
A na\"ive approach is to run a stochastic exp-concave optimization method on the convex hull of $\F$, but this suffers an excess risk bound scaling as $|\F|$ rather than $\log |\F|$. 
We instead start with an initial procedure that drastically reduces the set of candidates to a set of $O(\log(1/\delta)$. To this end, note that an online-to-batch conversion of the progressive mixture rule run on $n$ samples obtains expected excess risk at most $\log |\F| / (\eta (n + 1))$. Hence, $K$ independent runs yield a hypothesis with the same bound inflated by a factor $e$ with probability at least $1 - e^{-K}$ (we assume that this high probability event holds hereafter). 
At this point, it seems that we have replaced the original problem with an isomorphic one, as we do not know which $j \in [K]$ yields the desired candidate $\hat{f}_j$, and the corresponding subclass is still clearly non-convex. However, by taking the convex hull of this set of $K$ predictors and reparameterizing the problem, we arrive at a stochastic $\eta$-exp-concave optimization problem over the $K$-dimensional simplex; the best predictor in the convex hull clearly at least as good as the best one in $\F$. Thus, our analyses of EWOO and \textsc{ConfidenceBoost} apply and the results follow.

\section{Discussion and Open Problems}
\label{sec:discussion}

We presented the first high probability $O(d / n)$ excess risk bound for exp-concave statistical learning. 
The key to proving this bound was the connection between exp-concavity and the central condition, a connection which suggests that exp-concavity implies a \emph{low noise} condition. 
Here, low noise can be interpreted either in terms of 
    the central condition, by the exponential decay of the negative tail of the excess loss random variables, 
    or in terms of the Bernstein condition, by the variance of the excess loss of a hypothesis $f$ being controlled by its excess risk. 
All our results for stochastic exp-concave optimization were based on this low noise interpretation of exp-concavity. In contrast, 
The previous in-expectation $O(d/n)$ results of \cite{koren2015fast} and \cite{gonen2016tightening} used the geometric/convexity-interpretation of exp-concavity, which we further boosted to high probability results using the low noise interpretation. It would be interesting to get a high probability $O(d/n)$ result that proceeds purely from a low noise interpretation or purely from a geometric/convexity one.

Many results flowing from algorithmic stability often only yield in-expectation bounds, with high probability bounds stemming either from \emph{(i)}  
a posthoc confidence boosting procedure --- typically involving Hoeffding's inequality, which ``slows down'' fast rate results; 
or \emph{(ii)} quite strong stability notions --- e.g.~uniform stability allows one to apply McDiarmid's inequality to a single run of the algorithm \citep{bousquet2002stability}. 
Is it a limitation of algorithmic stability techniques that high probability $O(d/n)$ fast rates seem to be out of reach without a posthoc confidence boosting procedure, or are we simply missing the right perspective? One reason to avoid a confidence boosting procedure is that the resulting bounds suffer from a multiplicative $\log(1/\delta)$ factor rather than the lighter effect of an additive $\log(1/\delta)$ factor in bounds like Theorem \ref{thm:erm-whp-bound}.  
As we mentioned earlier, we conjecture that the basic ERM method obtains a high probability $O(d/n)$ rate, and a potential path to show this rate would be to control a localized complexity as done by \cite{sridharan2009fast} but using a more involved argument based on exp-concavity rather than strong convexity.

We also developed high probability quantile-like risk bounds for model selection aggregation, one with an optimal rate and another with a slightly suboptimal rate but no explicit dependence on the Lipschitz continuity of the loss. However, our bound form is not yet a full quantile-type bound; it degrades when the \textsc{gap} term is large, while the bound of \cite{lecue2014optimal} does not have this problem. Yet, our bound provides an improvement when there is a neighborhood around $f^*$ with large prior mass, which the bound of \citeauthor{lecue2014optimal} cannot do. It is an open problem to get a bound with the best of both worlds.


\bibliography{stochmix_confidence_boost}

\appendix

\section{Proofs for Stochastic Exp-Concave Optimization}
\label{sec:proofs}

\begin{proof}[of Lemma \ref{lemma:exp-concavity-to-central}]
The exp-concavity of $f \mapsto \loss(f, z)$ for each $z \in \Z$ implies that, for all $z \in \Z$ and all distributions $Q$ over $\F$:
\begin{align*}
\E_{f \sim Q} \left[ e^{-\eta \loss(f, z)} \right] \leq 
e^{-\eta \loss(\E_{f \sim Q} [ f ], z)} 
\quad \Longleftrightarrow \quad 
\loss(\E_{f \sim Q} [ f ], z) 
\leq -\frac{1}{\eta} \log \E_{f \sim Q} \left[ e^{-\eta \loss(f, z)} \right] .
\end{align*}
It therefore holds that for all distributions $P$ over $\Z$, for all distributions $Q$ over $\F$, there exists (from convexity of $\F$) $f^* = \E_{f \sim Q} [ f ] \in \F$ satisfying
\begin{align*}
\E_{Z \sim P} [ \loss(f^*, Z) ] 
\leq \E_{Z \sim P} \left[ -\frac{1}{\eta} \log \E_{f \sim Q} \left[ e^{-\eta \loss(f, Z)} \right] \right] . \label{eqn:stoch-mix}
\end{align*}
This condition is equivalent to \emph{stochastic mixability} as well as the \emph{pseudoprobability convexity (PPC) condition}, both defined by \Citet{vanerven2015fast}. To be precise, for stochastic mixability, in Definition 4.1 of \Citet{vanerven2015fast}, take their $\F_d$ and $\F$ both equal to our $\F$, their $\mathcal{P}$ equal to $\{P\}$, and $\psi(f) = f^*$; then strong stochastic mixability holds. Likewise, for the PPC condition, in Definition 3.2 of \Citet{vanerven2015fast} take the same settings but instead $\phi(f) = f^*$; then the strong PPC condition holds. Now, Theorem 3.10 of \Citet{vanerven2015fast} states that the PPC condition implies the (strong) central condition.
\end{proof}

\begin{proof}[of Theorem \ref{thm:erm-whp-bound}]
First, from Lemma \ref{lemma:exp-concavity-to-central}, the convexity of $\F$ together with $\eta$-exp-concavity implies that $(P, \loss, \F)$ satisfies the $\eta$-central condition.

The remainder of the proof is a drastic simplification of the proof of Theorem 7 of \cite{mehta2014stochastic}. Technically, Theorem 7 of that works applies directly, but one can get substantially smaller constants by avoiding much of the technical machinery needed there to handle VC-type classes (e.g.~symmetrization, chaining, Talagrand's inequality).

Denote by $\xsloss := \loss_f - \loss_{f^*}$ the excess loss with respect to comparator $f^*$. 
Our goal is to show that, with high probability, ERM does not select any function $f \in \F$ whose excess risk $\E [ \xsloss ]$ is larger than $\frac{a}{n}$ for some constant $a$. Clearly, with probability 1 ERM will never select any function for which both $\xsloss \geq 0$ almost surely and with some positive probability $\xsloss > 0$; we call these functions the empirically inadmissible functions. For any $\gamma_n > 0$, let $\F_{\succeq \gamma_n}$ be the subclass formed by starting with $\F$, retaining only functions whose excess risk is at least $\gamma_n$, and further removing the empirically inadmissible functions.

Our goal now may be expressed equivalently as showing that, with high probability, ERM does not select any function $f \in \F_{\succeq \gamma_n}$ where $\gamma_n = \frac{a}{n}$ and $a > 1$ is some constant to be determined later. 
Let $\F_{\succeq \gamma_n,\varepsilon}$ be an optimal proper $(\varepsilon / \Lip)$-cover for $\F_{\succeq \gamma_n}$ in the $\ell_2$ norm. From the Lipschitz property of the loss it follows that this cover induces an $\varepsilon$-cover in sup norm over the loss-composed function class $\left\{ \loss_f : f \in \F_{\succeq \gamma_n} \right\}$. 
Observe that an $\varepsilon$-cover of $\F_{\succeq \gamma_n}$ in the $\ell_2$ norm has cardinality at most $(4 R / \varepsilon)^d$ \cite[equation 1.1.10]{carl1990entropy}, 
and the cardinality of an optimal \emph{proper} $\varepsilon$-cover is at most the cardinality of an optimal $(\varepsilon / 2)$-cover. \cite[Lemma 2.1]{vidyasagar2002learning}. 
It hence follows that $|\F_{\succeq \gamma_n,\varepsilon}| \leq \left( \frac{8 \Lip R}{\varepsilon} \right)^d$. 

Let us consider some fixed $f \in \F_{\succeq \gamma_n,\varepsilon}$. Since we removed the empirical inadmissible functions, there exists some $\eta_f \geq \eta$ for which $\E [ e^{-\eta_f \xsloss} ] = 1$. 
Theorem 3 and Lemma 4, both from \cite{mehta2014stochastic}, imply that
\begin{align*}
\log \E_{Z \sim P} \left[ e^{-(\eta_f / 2) \xsloss} \right] \leq -\frac{0.18 \eta_f a}{(B \eta_f \opmax 1) n} .
\end{align*}
Applying Theorem 1 of \cite{mehta2014stochastic} with $t = \frac{a}{2n}$ and the $\eta$ in that theorem set to $\eta_f / 2$ yields:
\begin{align*}
\Pr \left( \frac{1}{n} \sum_{j=1}^n \xsloss(Z_j) \leq \frac{a}{2n} \right) 
\leq \exp \left( -0.18 \frac{\eta_f}{B \eta_f \opmax 1} a + \frac{a \eta_f}{4 n} \right) .
\end{align*}
Taking a union bound over $\F_{\succeq \gamma_n,\varepsilon}$ and using $\eta \leq \eta_f$ for all $f \in \F_{\succeq \gamma_n,\varepsilon}$, we have that
\begin{align*}
\Pr \left( \exists f \in \F_{\succeq \gamma_n,\varepsilon} : 
    \frac{1}{n} \sum_{j=1}^n \xsloss(Z_j) \leq \frac{a}{2n} \right) 
\leq \left( \frac{8 \Lip R}{\varepsilon} \right)^d 
                 \exp \left( -0.18 \frac{\eta}{B \eta \opmax 1} a + \frac{a \eta}{4 n} \right) .
\end{align*}

Setting $\varepsilon = \frac{1}{2 n}$ and taking $n \geq 5$, from inversion it follows that with probability at least $1 - \delta$, for all $f \in \F_{\succeq \gamma_n, \varepsilon}$, we have $\frac{1}{n} \sum_{j=1}^n \xsloss(Z_j) \leq \frac{a}{2n}$, where
\begin{align*}
a = 8 \left( B \opmax \frac{1}{\eta} \right) \left( d \log(16 \Lip R n) + \log \frac{1}{\delta} \right) .
\end{align*}
Now, since $\sup_{f \in \F_{\succeq \gamma_n}} \min_{f_\varepsilon \in \F_{\succeq \gamma_n,\varepsilon}} \|\loss_f - \loss_{f_\varepsilon}\|_\infty \leq \frac{1}{2 n}$, and increasing $a$ by 1 to guarantee that $a > 1$, with probability at least $1 - \delta$, for all $f \in \F_{\succeq \gamma_n}$, we have $\frac{1}{n} \sum_{j=1}^n \xsloss(Z_j) > 0$. 
\end{proof}

\begin{proof}[of Lemma \ref{lemma:central-to-bernstein}]
The main tool we use is part 2 of Theorem 5.4 of \Citet{vanerven2015fast}. First, as per the proof of Lemma \ref{lemma:exp-concavity-to-central}, note that the central condition as defined in the present work is equivalent to the strong PPC condition of \Citet{vanerven2015fast}. We actually can improve that result due to our easier setting because we may take their function $v$ to be the constant function identically equal to $\eta$. Consequently, in equation (70) of \Citet{vanerven2015fast}, we may take $\varepsilon = 0$, improving their constant $c_2$ by a factor of 3; moreover, their result actually holds for the second moment, not just the variance, yielding:
\begin{align} \label{eqn:ugly-bernstein}
\E [ X^2 ] 
\leq \frac{2}{\eta \kappa(-2 \eta B)} \E [ X ] ,
\end{align}
where $\kappa(x) = \frac{e^x - x - 1}{x^2}$. 

We now study the function 
\begin{align*}
x \mapsto \frac{1}{\kappa(-x)} = \frac{x^2}{e^{-x} + x - 1} .
\end{align*}
We claim that for all $x \geq 0$:
\begin{align*}
\frac{x^2}{e^{-x} + x - 1} \leq 2 + x .
\end{align*}
L'H\^opital's rule implies that the inequality holds for $x = 0$, and so it remains to consider the case of $x > 0$. 

First, observe that the denominator is nonnegative, and so we may rewrite this inequality as
\begin{align*}
x^2 \leq (2 + x) (e^{-x} + x - 1) ,
\end{align*}
which simplifies to
\begin{align*}
0 \leq 2 e^{-x} + x + x e^{-x} - 2 
\qquad \Leftrightarrow \qquad 
2 (1 - e^{-x}) \leq x (1 + e^{-x}) .
\end{align*}

Therefore, we just need to show that, for all $x > 0$,
\begin{align*}
\frac{2}{x} 
\leq \frac{1 + e^{-x}}{1 - e^{-x}} 
= \frac{e^{x/2} + e^{-x/2}}{e^{x/2} - e^{-x/2}}
= \coth(x/2) ,
\end{align*}
which is equivalent to showing that for all $x > 0$,
\begin{align*}
\tanh(x) \leq x .
\end{align*}
But this indeed holds, since
\begin{align*}
\tanh(x) 
= \frac{e^x - e^{-x}}{e^x + e^{-x}} 
&= \frac{2 (x + \frac{x^3}{3!} + \frac{x^5}{5!} + \ldots)}
              {2 (1 + \frac{x^2}{2!} + \frac{x^4}{4!} + \ldots)} \\
&= x \cdot 
     \frac{1 + \frac{x^2}{3!} + \frac{x^4}{5!} + \ldots}
             {1 + \frac{x^2}{2!} + \frac{x^4}{4!} + \ldots} \\
&\leq x .
\end{align*}
The desired inequality is now established.

Returning to \eqref{eqn:ugly-bernstein}, we have
\begin{align*}
\E [ X^2 ] 
\leq \frac{2}{\eta} (2 + 2 \eta B) \E [ X ] 
\leq 4 \left( \frac{1}{\eta} + B \right) \E [ X ] .
\end{align*}
\end{proof}

\begin{proof}[of Lemma \ref{lemma:erm-bernstein}]
The following simple version of Bernstein's inequality will suffice for our analysis. Let $X_1, \ldots, X_n$ be independent random variables satisfying $X_j \geq B$ almost surely.
Then
\begin{align*}
\Pr \left( \frac{1}{n} \sum_{j=1}^n X_j - \E [ X ] \geq t \right) 
\leq 
  \exp \left( 
      -\frac{n t^2}{2 \left( \E \left[ \frac{1}{n} \sum_{j=1}^n X^2 \right] + \frac{B t}{3} \right)}
  \right) .
\end{align*}

Denote by $\xsloss := \loss_f - \loss_{f_1}$ the excess loss with respect to comparator $f_1$. 
Fix some $f \in \G' \setminus \{f_1\}$, take $X = -\xsloss$, and set $t = \E [ \xsloss ]$, yielding:
\begin{align*}
\Pr \left( \frac{1}{n} \sum_{j=1}^n \xsloss(Z_j) \leq 0 \right) 
&\leq 
  \exp \left( 
      -\frac{n \E [ \xsloss ]^2}{2 ( \E [ \xsloss^2 ] + \frac{1}{3} B \E [ \xsloss ] )}
  \right) \\
&\leq
  \exp \left( 
      -\frac{n \E [ \xsloss ]^2}{2 ( C \E [ \xsloss ]^q + \frac{1}{3} B \E [ \xsloss ] )}
  \right) \\
&=
  \exp \left( 
      -\frac{n \E [ \xsloss ]^{2 - q}}{2 \left( C + \frac{1}{3} B \E [ \xsloss ]^{1 - q} \right)}
  \right) \\
&\leq
  \exp \left( 
      -\frac{n \E [ \xsloss ]^{2 - q}}{2 \left( C + \frac{1}{3} B^{2 - q} \right)}
  \right) .
\end{align*}

Therefore, if
\begin{align} \label{eqn:bernstein-level}
\E [ \xsloss ] \geq \left( \frac{2 \left( C + \frac{B^{2 - q}}{3} \right) \log \frac{|\G'|}{\delta}}{n} \right)^{1/(2 - q)} ,
\end{align}
then it holds with probability at least $1 - \frac{\delta}{|\G'| - 1}$ that $\frac{1}{n} \sum_{j=1}^n \xsloss(Z_j) > 0$. The result follows by taking a union bound over the subclass of $\G' \setminus \{f_1\}$ for which \eqref{eqn:bernstein-level} holds.
\end{proof}

\section{Proofs for Model Selection Aggregation (Section \ref{sec:ms-aggregate})}

\begin{proof}[of Theorems \ref{thm:pm-ewoo} and \ref{thm:pm-cb}]
The starting point is the following bound for the progressive mixture rule when run with prior $\pi$ and parameter $\eta$, due to Audibert (see Theorem 1 of \cite{audibert2008progressive}, but the result was already proved in an earlier technical report version of \cite{audibert2009fast} (see Corollary 4.1 and Lemma 3.3 therein). 
When run on an $n$-sample, an online-to-batch conversion of the progressive mixture rule yields a hypothesis $\hat{f}$ satisfying
\begin{align*}
\E_{Z^n} \left[ \E_Z \left[ \loss(Y, \hat{f}(X)) \right] \right] 
\leq \inf_{\rho \in \Delta(\F)} \left\{ 
              \E_{f \sim \rho} \E_Z \left[ \loss(Y, f(X)) \right] 
              + \frac{D(\rho \pipes \pi)}{\eta (n + 1)}
          \right\} 
\end{align*}
where $D(\rho \pipes \pi)$ is the KL-divergence of $\rho$ from $\pi$.\footnote{We say ``of $\rho$ from $\pi$'' because the Bregman divergence form of the KL-divergence, which makes clear that the KL-divergence is measure of the curvature of negative Shannon entropy between $\rho$ and $\pi$ when considering a first-order Taylor expansion around $\pi$.} 
Note that this bound does not explicitly depend on the boundedness nor the Lipschitz continuity of the loss. 

Fix some $\rho^*$ that nearly obtains the infimum (or obtains it, if possible). 
Then
\begin{align*}
\E_{Z^n} \left[ \E_Z \left[ \loss(Y, \hat{f}(X)) \right] \right] 
- \E_{f \sim \rho^*} \E_Z \left[ \loss(Y, f(X)) \right] 
\leq \frac{D(\rho^* \pipes \pi)}{\eta (n + 1)} .
\end{align*}

We cannot apply the boosting the confidence trick just yet as the LHS is not a nonnegative random variable; this issue motivates the following rewrite.
\begin{align*}
&\E_{Z^n} \left[ \E_Z \left[ \loss(Y, \hat{f}(X)) \right] \right] 
   - \E_Z \left[ \loss(Y, f^*(X)) \right] \\
&\leq \underbrace{\E_{f \sim \rho^*} \E_Z \left[ \loss(Y, f(X)) \right] 
          - \E_Z \left[ \loss(Y, f^*(X)) \right]}_{\gap(\rho^*, f^*)} 
          + \frac{D(\rho^* \pipes \pi)}{\eta (n + 1)} .
\end{align*}

When the progressive mixture rule is run on $K$ independent samples, yielding hypotheses $f^{(1)}, \ldots, f^{(K)}$, then Markov's inequality implies that with probability at least $1 - e^{-K}$ (over the $(K n)$-sample) there exists $j \in [K]$ for which 
\begin{align*}
&\E_Z \left[ \loss(Y, f^{(j)}(X)) \right] - \E_Z \left[ \loss(Y, f^*(X)) \right] \\
&\leq e \left( 
                 \gap(\rho^*, f^*) 
                 + \frac{D(\rho^* \pipes \pi)}{\eta (n + 1)} 
             \right) ,
\end{align*}
which can be re-expressed as
\begin{align*}
&\E_Z \left[ \loss(Y, f^{(j)}(X)) \right] - \E_Z \left[ \loss(Y, f^*(X)) \right] \\
&\leq e \cdot \gap(\rho^*, f^*) 
          + \frac{e \cdot D(\rho^* \pipes \pi)}{\eta (n + 1)} \\
&= e \left( 
          \inf_{\rho \in \Delta(\F)} \left\{ 
              \E_{f \sim \rho} \E_Z \left[ \loss(Y, f(X)) \right] 
              + \frac{D(\rho \pipes \pi)}{\eta (n + 1)}
          \right\} 
           - \E_Z \left[ \loss(Y, f^*(X)) \right] 
      \right) \\
&= e \cdot \BR{n} .
\end{align*}

In the sequel, we assume that this high probability event has occurred.

Now, let $\widetilde{\F} = \conv \left( \{f^{(1)}, \ldots, f^{(K)}\} \right)$. 
Clearly, $f^{(j)} \in \widetilde{\F}$, and so we also have
\begin{align} \label{eqn:gap-bound}
\inf_{f \in \widetilde{\F}} \E_Z \left[ \loss(Y, f(X)) \right] 
\leq e \cdot \BR{n} .
\end{align}

It therefore is sufficient to learn over $\widetilde{\F}$ and compete with its risk minimizer. But this is only a $K$-dimensional problem, and if $\delta = e^{-K}$, we have $K = \log \frac{1}{\delta}$. To see why the problem is only $K$-dimensional, consider the transformed problem, where
\begin{align*}
\tilde{x} = \begin{pmatrix} f^{(1)}(x) \\ \vdots \\ f^{(K)}(x) \end{pmatrix} .
\end{align*}

The loss can now be reparameterized, from
\begin{align*}
\loss \colon \widetilde{\F} \to \reals
\quad \text{with} \quad
\loss \colon f \mapsto \loss(y, f(x))
\end{align*}
\begin{flalign*}
\text{to} && 
\tilde{\loss} \colon \Delta^{K-1} \to \reals
\quad \text{with} \quad
\tilde{\loss} \colon q \mapsto \loss(y, \langle q, \tilde{x} \rangle) , &&
\end{flalign*}
where $\Delta^{K-1}$ is the $(K - 1)$-dimensional simplex 
$\left\{ q \in [0,1]^K \colon \sum_{j=1}^K q_j = 1 \right\}$.

$\Delta^{K-1}$ is clearly convex and the loss is $\eta$-exp-concave with respect to $q \in \Delta^{K-1}$; to see the latter, observe that from the $\eta$-exp-concavity of the loss with respect to $\hat{y} = \langle q, \tilde{x} \rangle$:
\begin{align*}
\E_{q \sim P_q} \left[ e^{-\eta \loss(y, \langle q, \tilde{x} \rangle)} \right]
&\leq e^{-\eta \loss(y, \E_{q \sim P_q} \left[ \langle q, \tilde{x} \rangle \right])} \\
&= e^{-\eta \loss(y, \langle \E_{q \sim P_q} [ q ], \tilde{x} \rangle)} .
\end{align*}
Lastly, the loss is still bounded by $B$ since $\widetilde{\F}$ consists only of convex aggregates of $\hat{f}_1, \ldots, \hat{f}_K$, themselves convex aggregates over $\F$ (and we assumed boundedness of the loss with respect to the original class).

We now can proceed in two ways. The high probability bound for EWOO (the first display after Corollary \ref{cor:online-to-batch}) applies immediately. This bound can be simplified to (taking $d = K = \lceil \log(2/\delta) \rceil$)
\begin{align*}
O \left( 
  \frac{\sqrt{B} \left( \log \frac{1}{\delta} + \sqrt{\log \frac{1}{\delta}} \log n \right)}
         {\eta n} 
  + \frac{B \left( \log \log n + \log \frac{1}{\delta} \right)}{n} 
\right) ,
\end{align*}
which, in light of \eqref{eqn:gap-bound}, proves Theorem \ref{thm:pm-ewoo}.

If we further assume the loss framework of \cite{gonen2016tightening}, then $\tilde{\loss}$ still satisfies $\alpha$-strong convexity in the sense needed because, conditional on the actual prediction $\hat{y}$, the loss $\tilde{\loss}$ is the same as loss $\loss$. Hence, the bound \eqref{eqn:boosted-gonen-bound} \textsc{ConfidenceBoost} from Corollary \ref{cor:whp-bounds} applies (taking $d = K = \lceil \log \frac{3}{\delta} \rceil$), finishing the proof of Theorem \ref{thm:pm-cb}.
\end{proof}

\end{document}